\documentclass[preprint,11pt,authoryear]{elsarticle}

\usepackage{amssymb}
\usepackage{amsthm}

\journal{Theoretical Computer Science}

\usepackage{latexsym,amsmath,graphicx,float}
\newcommand{\shrink}[1]{}

\usepackage{algorithm,algorithmic}

\newcommand{\err}{\operatorname{err}}

\newcommand{\reg}{\operatorname{reg}}
\newcommand{\wreg}{\operatorname{creg}}

\newcommand{\leaves}{L}

\renewcommand{\v}{{\rho}}
\date{}

\newcommand{\qedblob}{\mbox{\rule[-1.5pt]{5pt}{10.5pt}}}
\renewenvironment{proof}{\noindent{\bf Proof:~}}{\qedblob}

\renewcommand{\algorithmicrequire}{\textbf{Input:}}
\renewcommand{\algorithmicensure}{\textbf{Notation:}}
\renewcommand{\algorithmicforall}{\textbf{for each}}

\newcommand{\I}{\mathbf{1}}
\newcommand{\E}{\mathbf{E}}

\newcommand{\cX}{{X}}
\newcommand{\cY}{{Y}}

\theoremstyle{plain}
\newtheorem{theorem}{Theorem}
\theoremstyle{plain}
\newtheorem{lemma}[theorem]{Lemma}
\theoremstyle{plain}
\newtheorem{corollary}[theorem]{Corollary}

\floatstyle{ruled}
\newfloat{algorithm}{tbp}{loa}
\floatname{algorithm}{Algorithm}

\begin{document}

\begin{frontmatter}



\title{Error-Correcting Tournaments}

\author{Alina Beygelzimer}
\address{IBM Thomas J. Watson Research Center, Hawthorne, NY 10532}
\ead{beygel@us.ibm.com}
\author{John Langford}
\address{Yahoo! Research, New York, NY 10018}
\ead{jl@yahoo-inc.com}
\author{Pradeep Ravikumar}
\address{Department of Computer Sciences, University of Texas, Austin, TX 78712}
\ead{pradeepr@cs.utexas.edu}


\begin{abstract}
  Text of abstract We present a family of pairwise tournaments
  reducing $k$-class classification to binary classification. These
  reductions are provably robust against a constant fraction of binary
  errors, simultaneously matching the best possible computation
  $O(\log k)$ and regret $O(1)$.

  The construction also works for robustly selecting the best of
  $k$-choices by tournament.  We strengthen previous results by
  defeating a more powerful adversary than previously addressed while
  providing a new form of analysis.  In this setting, the error
  correcting tournament has depth $O(\log k)$ while using $O(k \log
  k)$ comparators, both optimal up to a small constant.
\end{abstract}


\end{frontmatter}


\section{Introduction}
We consider the classical problem of multiclass classification,
where given an instance $x\in X$, the goal is to predict 
the most likely label $y\in \{1,\ldots,k\}$, according to some unknown 
probability distribution.

A common general approach to multiclass learning is to
reduce a multiclass problem to a set of binary classification 
problems~\cite{ASS,ECOC,GS,All-Pairs,SECOC}.
This approach is composable with any binary learning algorithm,
including online algorithms, Bayesian algorithms, and even humans.

\shrink{
An alternative is to design a multiclass learning algorithm directly, 
typically by extending an existing algorithm for binary classification.
A difficulty with this direct approach is that 
some algorithms cannot be easily modified to handle a different
learning problem.
For example, the first and still commonly used multiclass versions
of the support vector machine may not even
converge to the best possible predictor no matter how
many examples are used (see \cite{Wahba}).
A single reduction yields a number of different multiclass algorithms
in this way.

%This paper is about an improvement in the reduction approach.
A key technique for analyzing reductions is \emph{regret analysis}, bounding
the regret of the resulting multiclass learner in terms of the 
regret of the binary classifiers on the binary problems.
Informally, regret is the
difference in loss between the predictor and the best possible
predictor on the same problem.  
Regret analysis is more refined than
loss analysis as it bounds only avoidable, excess loss, thus the bounds remain
meaningful for problems with high conditional noise.
}

A key technique for analyzing reductions
is \emph{regret analysis}, which
bounds the regret of the resulting multiclass classifier in terms of
the average classification regret on the induced binary problems.
Here {\it regret} (formally defined in Section~\ref{notation}) is the
difference between the incurred loss and the smallest
achievable loss on the problem, i.e.,
excess loss due to suboptimal prediction.

The most commonly applied reduction is
one-against-all, which creates a binary classification problem for each of
the $k$ classes. The classifier for class $i$ is trained to predict whether
the label is $i$ or not; predictions are then done by evaluating each binary
classifier and randomizing over those that predict ``yes,''
or over all labels if all answers are ``no''.

This simple reduction 
is {\it inconsistent}, in the sense that given
optimal (zero-regret) binary classifiers, the reduction may not 
yield an optimal multiclass classifier in the presence of noise. 
Optimizing squared loss of the binary predictions
instead of the $0/1$ loss makes the approach consistent, but the resulting 
multiclass regret scales as $\sqrt{2kr}$ in the worst case, where $r$ is
the average squared loss regret on the induced problems.
The Probing reduction~\cite{Probing} upper bounds $r$ by 
the average binary classification 
regret.  This composition gives a consistent reduction to binary classification,
but it has a square root dependence on the binary regret 
(which is undesirable as regrets are between 0 and 1).

The probabilistic error-correcting output code 
approach (PECOC)~\cite{SECOC} reduces $k$-class classification to learning
$O(k)$ regressors on the interval $[0,1]$, creating $O(k)$ binary examples
per multiclass example at both training and test time, with a test time
computation of $O(k^2)$. 
The resulting multiclass regret is bounded by $4\sqrt{r}$, removing
the dependence on the number of classes $k$.
When only a constant number of labels have non-zero probability given features,
the computation can be reduced to $O(k \log k)$ per example~\cite{CS}.

%The regression problems can be further reduced to binary classification
%using the Probing reduction~\cite{Probing}.
%The resulting multiclass regret is bounded by $4\sqrt{r}$, 
%where $r$ is the average square error regret, removing the dependence 
%on the number of classes $k$.
\medskip
This state of the problem raises several questions:
\begin{enumerate}
\item Is there a consistent reduction from multiclass to binary classification 
  that does not have a square root dependence on $r$~\cite{Bob}? 
  For example, an average binary regret of just $0.01$ may imply
  a PECOC multiclass regret of $0.4$.  %At the level of
  %units of regret, the equation is odd as the square root of one
  %regret bounds another.
\item Is there a consistent reduction 
  that requires just $O(\log k)$ computation, 
  matching the information theoretic lower bound? 

\smallskip
  The well-known $O(\log k)$ tree reduction distinguishes between 
  the labels using a balanced binary tree, with each non-leaf
  node predicting
  ``Is the correct multiclass label to the left or not?''~\cite{Tree}.  
  As shown in Section~\ref{tree}, this method is inconsistent.
% No lesser amount of computation is
%  possible, since any multiclass prediction must specify $O(\log k)$
%  bits. 
%  A trivial answer is ``yes'' using a divide and conquer
%  approach on a $O(\log k)$ depth tree on the labels.
%
% where a node predicts ``Is the label in
%  this subset or that subset?''  However, this approach is
%  inconsistent (see section~\ref{tree}), and hence invalid for our
%  notion of a learning reduction.
\item Can the above be achieved with a reduction that
  only performs pairwise comparisons between classes?  

\smallskip
  One fear associated with the PECOC approach is that it creates binary
  problems of the form ``What is the probability that the label is in
  a given random subset of labels?,''  which may be hard to solve.  Although
  this fear is addressed by regret analysis (as the latter operates only on
  avoidable, excess loss), and is 
  overstated in some cases~\cite{hunch_post,CS}, 
  it is still of some concern, especially with larger values of $k$.
  %(Regret analysis bounds only avoidable loss, thus the bounds remain
  %useful even for problems with high conditional noise.)
\end{enumerate}
The error-correcting tournament family presented here answers all
of these questions in the affirmative.  It provides an exponentially
faster in $k$ method for multiclass prediction with the resulting multiclass
regret bounded by $5.5 r$, where $r$ is the average binary regret; and
every binary classifier logically compares two distinct class labels.
%When we allow $O(k)$ computation, corresponding to correcting $k/2$ 
%binary errors, the resulting regret bound improves to $4r$. 

The result is based on a basic observation that if a non-leaf 
node fails to predict its binary label, which may
be unavoidable due to noise in the distribution, 
nodes between this node and the root should have
no preference for class label prediction.  Utilizing this observation,
we construct a reduction, called the \emph{filter tree}, which
uses a $O(\log k)$ computation per multiclass example at both
training and test time, and whose multiclass regret is bounded by $\log k$
times the average binary regret.

The decision process of a filter tree, viewed bottom up, can be viewed as a
single-elimination tournament on a set of $k$ players.
%Generalizing the construction to multiple tournaments can improve
%its error correction.
Using multiple independent single-elimination tournaments is of no use as it
does not affect the \emph{average} regret of an
adversary controlling the binary classifiers.
Somewhat surprisingly,
it is possible to have $\log k$ complete single-elimination 
tournaments between $k$ players in $O(\log k)$ rounds,
with no player playing twice in the same round. % ~\cite{Min_find}.  
An \emph{error-correcting tournament}, 
first pairs labels in such
simultaneous single-elimination tournaments, followed by a final
carefully weighted single-elimination tournament that decides
among the $\log k$ winners of the first phase.  As
for the filter tree, test time evaluation can start at the
root and proceed to a multiclass label with $O(\log k)$
computation.

This construction is also useful for the problem of robust
search, yielding the first algorithm which allows the adversary 
to err a constant fraction of the time in the ``full lie'' setting~\cite{RGL},
where a comparator can missort any comparison.  Previous work either
applied to the ``half lie'' case where a comparator can fail to sort
but can not actively missort~\cite{Min_find,Yaos}, or to a ``full lie''
setting where an adversary has a fixed known bound on the number of
lies~\cite{RGL} or a fixed budget on the fraction of errors so
far~\cite{borgstrom,aslam}.  Indeed, it might even appear impossible
to have an algorithm robust to a constant fraction of full lie errors
since an error can always be reserved for the last comparison.  
Repeating the last comparison $O(\log k)$ times defeats this strategy.

The result here is also useful for the actual problem of tournament
construction in games with real players.  Our analysis does not assume
that errors are {\it i.i.d.}~\cite{feige}, or have known noise
distributions~\cite{Adler} or known outcome distributions given player
skills~\cite{TrueSkill}.  Consequently, the tournaments we construct
are robust against severe bias such as a biased referee or some forms
of bribery and collusion.  Furthermore, the tournaments we construct
are shallow, requiring fewer rounds than $m$-elimination bracket
tournaments, which do not satisfy the guarantee provided here.  In an
$m$-\emph{elimination bracket tournament}, bracket $i$ is a
single-elimination tournament on all players except the winners of
brackets $1,\ldots,i-1$.  After the bracket winners are determined,
the player winning the last bracket $m$ plays the winner of bracket
$m-1$ repeatedly until one player has suffered $m$ losses (they start
with $m-1$ and $m-2$ losses respectively).  The winner moves on to
pair against the winner of bracket $m-2$, and the process continues
until only one player remains.  This method does not scale well to
large $m$, as the final elimination phase takes $\sum_{i=1}^m i - 1 =
O(m^2)$ rounds.  Even for $k=8$ and $m=3$, our constructions have
smaller maximum depth than bracketed $3$-elimination.  To see that the
bracketed $m$-elimination tournament does not satisfy our goal, note
that the second-best player could defeat the first player in the first
single elimination tournament, and then once more in the final
elimination phase to win, implying that an adversary need control only
two matches.

\paragraph{Paper overview}
We begin by defining the basic concepts and introducing some of the 
notation in Section~\ref{notation}.
Section~\ref{tree} shows that the simple divide-and-conquer tree
approach is inconsistent, motivating the
Filter Tree algorithm described in section~\ref{S:algorithm} (which applies to
more general cost-sensitive multiclass problems).
Section~\ref{S:analysis} proves that the algorithm has the best
possible computational dependence, and gives two upper bounds on the regret
of the returned (cost-sensitive) multiclass classifier.
%bounding the multiclass regret 
%by $\log k$ times the binary regret, and 
%the other bounding the cost-sensitive multiclass regret by $k/2$ times
%the binary regret.
Subsection~\ref{S:experiments} presents some
experimental evidence that the Filter Tree is indeed a practical
approach for multiclass classification.

Section~\ref{sec:Multi-Elimination} presents the error-correcting
tournament family parametrized by an integer $m \geq 1$,
which controls the tradeoff between maximizing robustness ($m$ large)
and minimizing depth ($m$ small).  Setting $m=1$ gives the Filter
Tree, while $m = 4 \ln k$ gives a (multiclass to binary) 
regret ratio of $5.5$ with $O(\log k)$
depth. Setting $m=c k$ gives regret ratio of $3 + O(1/c)$ with depth
$O(k)$.  The results here provide a nearly free
generalization of earlier work~\cite{Min_find} in the robust search
setting, to a more powerful adversary that can missort as well as fail
to sort. %, and which is only charged according to two labels conditional
%probability difference.
%

Section~\ref{sec:LB} gives an algorithm independent lower bound of 2 on
the regret ratio for large $k$.  When the number of calls to a binary
classifier is independent (or nearly independent) of the label
predicted, we strengthen this lower bound to $3$ for large $k$.

\section{Preliminaries}
\label{notation}
Let $D$ be the underlying distribution over
$\cX \times \cY$, where $\cX$ is some observable feature
space and $\cY = \{1,\ldots,k\}$ is the label space.
The \emph{error rate} of a classifier $f:  \cX \rightarrow \cY$
on $D$ is given by
\[
\err(f,D) = \mathbf{Pr}_{(x,y)\sim D} [f(x) \neq  y].
\]
The \emph{multiclass regret} of $f$ on $D$ is defined as
\[
\reg(f,D) = \err(f,D) - \min\limits_{g: \cX \rightarrow \cY} \err(g,D).
\]
The algorithms here extend to the \emph{cost-sensitive} 
case, where the underlying distribution $D$ is over $\cX\times[0,1]^{k}$.
The \emph{expected cost} of a classifier $f: \cX\rightarrow \cY$
on $D$ is
\[
\ell(f,D) = \mathbf{E}_{(x,c)\sim D}\left[c_{f(x)}\right].
\]
Here $c \in [0,1]^{k}$
gives the cost of each of the $k$ choices for $x$.
As in the multiclass case, the \emph{cost-sensitive regret} 
of $f$ on $D$ is defined as
\[
\wreg(f,D) = \ell(f,D) - \min_{g: \cX \rightarrow \cY} \ell(g,D).
\]

\section{Inconsistency of Divide and Conquer Trees}
\label{tree}
One standard approach for reducing multiclass learning to binary learning
is to split the set of labels in half, learn
a binary classifier to distinguish between the two subsets, and repeat
recursively until each subset contains one label.  Multiclass
predictions are made by following a chain of classifications from the root
down to the leaves.

This tree reduction transforms $D$ into a distribution $D_T$ 
over binary labeled examples
by drawing a multiclass example $(x,y)$ from $D$ and
a random non-leaf node $i$, 
and outputting instance $\langle x,i\rangle$ with label 1 if $y$ is in the 
left subtree of node $i$, and 0 otherwise.
A binary classifier $f$ for this induced problem gives a multiclass 
classifier $T(f)$, via a chain 
of binary predictions starting from the root.

The following theorem gives an example of a multiclass problem such
that even if we have an optimal classifier for the induced
binary problem at each node, the tree reduction does not
yield an optimal multiclass predictor. %The proof is constructive, and
%the intuition closely follows a similar theorem for error correcting
%output codes~\cite{SECOC}.  

\begin{theorem}
For all $k \geq 3$, for all binary trees over the labels, there exists
a multiclass distribution $D$ such that 
$\ \reg(T(f^*),D) > 0$
for any $f^* = \arg \min\limits_f \err(f,D_T)$.
\end{theorem}
\begin{proof}
Find a node with one subset corresponding to two labels and the other
subset corresponding to a single label. (If the tree is perfectly
balanced, simply let $D$ assign probability 0 to one of the labels.)
Since we can freely rename labels without changing the underlying
problem, let the first two labels be $1$ and $2$, and the third label
be $3$.

Fix any $\epsilon \in (0,1/12)$.
Choose $D$ with the property that labels $1$ and $2$
each have a $\frac{1}{4}+\epsilon$ chance of being drawn given $x$,
and label $3$ is drawn with the remaining probability of 
$\frac{1}{2}-2\epsilon$.
Under this distribution,
the fraction of examples for which label $1$ or $2$ is correct is $\frac{1}{2} +
2\epsilon$, so any minimum error rate binary predictor must choose either
label $1$ or label $2$.  Each of these choices has an error rate of
$\frac{3}{4}-\epsilon$.  The optimal multiclass predictor chooses label $3$ and
suffers an error rate of $\frac{1}{2}+2\epsilon$, implying that the regret 
of the tree classifier based on the optimal binary classifier is
$\frac{1}{4} - 3\epsilon$, which is strictly greater than 0 as $\epsilon < 1/12$.
\end{proof}

\section{The Filter Tree Algorithm}\label{S:algorithm}
\begin{figure}[t]
\centering
\includegraphics[width=.6\textwidth]{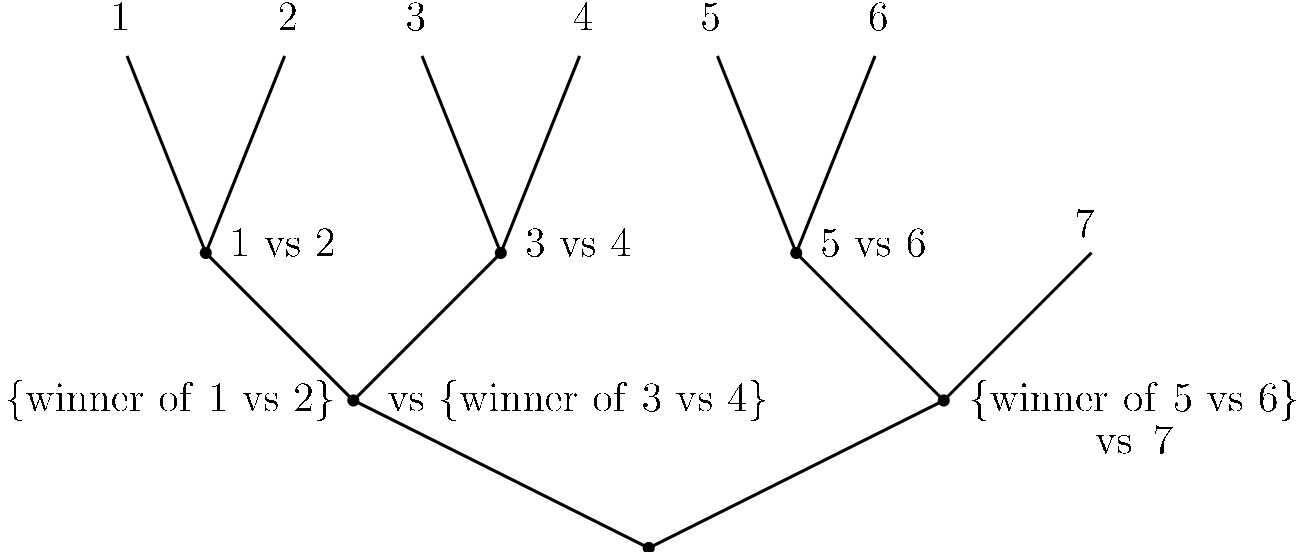}
\caption{\label{fig:filter-tree} 
Filter Tree.  Each node predicts whether the left 
or the right input label is more likely, conditioned on a given $x\in \cX$.
The root node predicts the best label for $x$.}
\end{figure}

The Filter Tree algorithm, illustrated by
Figure~\ref{fig:filter-tree},
is equivalent to a single-elimination tournament on the set of labels,
structured as a binary tree $T$ over the labels.
In the first round, the labels are paired according to the lowest
level of the tree, and a classifier is trained for each pair to
predict which of the two labels is more likely.  (The labels that don't
have a pair in a given round, win that round for free.)  The winning
labels from the first round are in turn paired in the second round,
and a classifier is trained to predict whether the winner of one pair
is more likely than the winner of the other.  The process of training
classifiers to predict the best of a pair of winners from the previous
round is repeated until the root classifier is trained.

\shrink{
The setting above is akin to Boosting:  At each round $t$,
a booster creates an input distribution $D_t$ and calls an
oracle learning algorithm to obtain a classifier with some error
$\epsilon_t$ on $D_t$.  The distribution $D_t$ depends on the classifiers
returned by the oracle in previous rounds.
The accuracy of the final classifier is analyzed in terms of $\epsilon_t$'s.
}

%In what follows, $T$ is a fixed binary tree over the labels.

The key trick in the training stage (Algorithm~\ref{alg:filter}) is to
form the right training set at each interior node.
We use $T_n$ to denote the subtree of $T$ rooted at node $n$, and
$\leaves(T)$ to denote the set of leaves in the tree $T$.
A training example for node
$n$ is formed {conditioned} on the predictions of classifiers in
the round before it.  Thus the learned classifiers from the first
level of the tree are used to ``filter'' the distribution over
examples reaching the second level of the tree.

Given $x$ and classifiers at each node, every edge in $T$ is identified 
with a unique label.
The optimal decision at any non-leaf node is to
choose the input edge (label) that is more likely according to the true
conditional probability.
This can be done by
using the outputs of classifiers in the round before it as a filter
during the training process: For each observation, we set the label to
0 if the left parent's output matches the multiclass label, 1 if the
right parent's output matches, and reject the example otherwise.

\algsetup{indent=2em}
\begin{algorithm}
\caption{{Filter tree training} (multiclass training set $S$, 
binary learning algorithm \texttt{Learn})}
\label{alg:filter}
\begin{algorithmic}
\medskip
\STATE
%$T_n$ is the subtree of $T$ rooted at node $n$.\\
Define $y_u=1$ if label $y$ is in the left subtree of node $u$; otherwise $y_u=0$. \\
\medskip
\FORALL{non-leaf node $n$ in order from leaves to root}
\STATE Set $S_n = \emptyset$ 
\FORALL{$(x,y) \in S$ such that $y\in \leaves(T_n)$ 
and all nodes $u$ on the path $n \leadsto y$ predict $y_{u}$ given $x$}
\STATE add $(x, y_n)$ to $S_n$
\ENDFOR
\STATE Let $f_n = \texttt{Learn}(S_n)$ 
\ENDFOR
\RETURN $f=\{ f_n \}$
\end{algorithmic}
\end{algorithm}

\begin{algorithm}
\caption{Cost-sensitive filter tree training (cost-sensitive training set $S$, 
importance weighted binary learner \texttt{Learn})}
\label{alg:cs-filter-tree} 
\begin{algorithmic}[1]
\medskip
\FORALL{non-leaf node $n$ in the order from leaves to root}
\STATE Set $S_n = \emptyset$ 
\FORALL{example $(x,c_1,...,c_k) \in S$} \label{line3}
\STATE
Let $a$ and $b$ be the two classes input to $n$ %(for internal nodes,
%these are the predictions of the left and the right subtrees on input $x$).\;
\STATE $S_n \leftarrow S_n \cup 
	\{ (x,\arg \min \{c_a,c_b\}, \underbrace{|c_a - c_b |}_{w_n(x,c)})\}$\vskip -.18in \label{line5}
\ENDFOR
\STATE
Let $f_n = \texttt{Learn}(S_n)$ 
\ENDFOR
\RETURN $f=\{f_n\}$
\end{algorithmic}
\end{algorithm}

\shrink{
\begin{algorithm}
\caption{Filter tree testing
(classifiers $\{ c_n \}$, test example $x\in \cX$)}
\label{alg:filter-tree-test}
\begin{algorithmic}
\RETURN
leaf $l$ such that every classifier on the path from
$l$ to the root prefers $l$
\end{algorithmic}
\end{algorithm}
}

Algorithm~\ref{alg:cs-filter-tree} extends this idea to the
cost-sensitive multiclass case where each choice has a different associated 
cost, as defined in Section~\ref{notation}.  
The algorithm relies upon an \emph{importance
weighted} binary learning algorithm \texttt{Learn}, which takes examples of the form
$(x,y,w)$, where $x\in X$ is a feature vector used for prediction, 
$y\in \{0,1\}$ is a binary label, and $w\in [0,\infty)$ is the importance 
any classifier pays if it doesn't predict $y$ on $x$. 
The importance weighted problem 
can be further reduced to binary classification using
the Costing reduction~\cite{costing}, which alters the underlying
distribution using rejection sampling on the importances.
This is the reduction we use here.

The testing algorithm is the same for both multiclass 
and cost-sensitive variants, and is very simple:
Given a test example $x\in X$, we output the label $y$ such that every
classifier on the path from the root to $y$ prefers $y$.

\section{Filter Tree Analysis}\label{S:analysis}
Before analyzing the regret of the algorithm, we note its computational
characteristics.

\subsection{Computational Complexity}
Since the algorithm is a reduction, 
we count the computational complexity of the reduction itself, 
assuming that oracle calls take unit time.

Algorithm~\ref{alg:filter} requires
$O(\log k)$ computation per multiclass example,
%This follows by
%noting that the algorithm can be implemented 
by searching for the correct leaf in
$O(\log k)$ time, then filtering back toward the root.
This matches the information theoretic lower bound 
since simply reading one of $k$ labels requires $\lceil \log_2 k\rceil$ bits.

Algorithm~\ref{alg:cs-filter-tree} requires $O(k)$
computation per cost-sensitive
example, because there are $k-1$ nodes, 
each requiring constant computation per example.
Since any method must read the $k$ costs, this bound is tight.

Testing
requires $O(\log k)$ computation per  
example to descend a binary tree.  Any method must write out 
$\lceil\log_2 k\rceil$ bits to specify its prediction.

\shrink{
\subsection{Regret Analysis}
Algorithm~\ref{alg:cs-filter-tree}
transforms cost-sensitive multiclass examples
into importance weighted binary labeled examples.
This process implicitly
transforms a distribution $D$ over cost-sensitive multiclass examples
into a distribution $D_{T}$ over importance weighted binary examples.

There are many induced problems, one for each call to the oracle
\texttt{Learn}.
To simplify the analysis, we can use a standard
transformation allowing us to consider only a single induced distribution:
add the node index $n$ as an additional feature into each
importance weighted binary example, and then train based upon the
union of all the training sets.  The learning algorithm produces a
single binary classifier $f(x,n)$ for which we can redefine
$f_n(x)$ as $f(x,n)$. The induced distribution $D_{T}$ can be defined
by the following process: (1) draw a cost-sensitive example $(x,c)$ from $D$, 
(2) pick a random node $n$, (3) create an importance-weighted sample 
according to the algorithm, except using $\langle x,n\rangle$ instead of $x$.
The induced distribution $D_{T}$ can be defined by the following process: (1)
draw a cost-sensitive example $(x,c)$ from $D$, (2) pick a random node $n$, 
(3) create an importance-weighted sample according to the
algorithm, except using $\langle x,n\rangle$ instead of $x$.

The theorem is quantified over {all} classifiers, and thus it holds
for the classifier returned by the algorithm.  In practice, one can 
either call the oracle multiple times to learn a separate classifier
for each node (as we do in our experiments), or use iterative
techniques for dealing with the fact that the classifiers are
dependent on other classifiers closer to the leaves.

The fact that classifiers are conditionally dependent on other
classifiers closer to the leaves is not problematic: In practice,
there are known effective iterative techniques for dealing with this
cyclic dependence; alternatively, a separate classifier can be trained
for each node.  In theory, we quantify over \emph{all} predictors,
which means that we can regard the learned predictor as fixed,
implying the induced problem is well defined.
}

\subsection{Regret Analysis}
Algorithm~\ref{alg:cs-filter-tree}
transforms each cost-sensitive multiclass example (line~\ref{line3})
into importance weighted binary labeled examples (line~\ref{line5}),
one for every non-leaf node $n$ in the tree.  This process implicitly
transforms the underlying 
distribution $D$ over cost-sensitive multiclass examples
into a distribution $D_n$ over importance weighted binary examples at each $n$.

We can further reduce from importance weighted binary classification 
to binary classification using
the Costing reduction~\cite{costing}, which
alters each $D_n$ using rejection sampling on the importance weights.
This composition further transforms $D_n$ into a distribution            
$D'_n$ over binary examples.\shrink{
Instead of considering the average binary regret over multiple induced 
problems, we can use a standard
transformation to consider only a single induced problem: 
add the node index $n$ as an additional feature into each
importance weighted binary example, and then train based upon the
union of all the training sets.  The learning algorithm produces a
single binary classifier $f(x,n)$ for which we can redefine
$f_n(x)$ as $f(x,n)$. The induced distribution can be defined
by the following process: (1) draw a cost-sensitive example $(x,c)$ from $D$,
(2) pick a random node $n$, (3) create an importance-weighted sample
according to the algorithm, except using $\langle x,n\rangle$ instead of $x$,
and rejection sample it according to Costing to remove the weight.
} 

\shrink{
We use the folk theorem from \cite{costing} saying that for all binary
classifiers $f$ and all importance weighted binary distributions $P$,
the importance weighted binary regret of $f$ on $P$ is upper bounded
by ${E}_{(x,y,w)\sim P} [w]$ 
times the binary regret of $f$ on the induced
binary distribution. % $\texttt{Costing}(P)$.
}

Let $f_n$ be a classifier for the binary classification problem
induced at node $n$.  The relevant quantity is the \emph{average binary
regret},
\begin{align}
\reg(f,D') = \frac{1}{\sum_{n\in T}W_n} 
\sum_{n\in T}\reg(f_n,D'_n)\,W_n,
\label{abr}
\end{align}
where $W_n = \E_{(x,c)\sim D}w_n(x,c)$, and 
$w_n(x,{c})$ is the importance weight formed in line~\ref{line5} of
Algorithm~\ref{alg:cs-filter-tree} (the difference in
cost between the two labels that node $n$ chooses between on $x$).
This quantity, which is just the average
\emph{importance weighted} binary regret of $f_n$ on $D_n$, 
is induced by the reduction (Algorithm~\ref{alg:cs-filter-tree}).

The core theorem below relates $\reg(f,D')$
to the regret of the resulting cost-sensitive classifier $T(f)$ on
$D$.  Again, given a test example $x\in X$, the classifier $T(f)$
returns the unique label $y$ such that every $f_n$ on the path from
the root to $y$ prefers $y$.

This type of analysis is similar to Boosting: At each round $n$, the
booster creates an input distribution $D_n$ and calls a weak learning
algorithm to obtain a classifier $f_n$, which has some error rate on
$D_n$.  The distribution $D_n$ depends on the classifiers returned by
the oracle in previous rounds.  The accuracy of the final classifier
on the original distribution $D$ is analyzed in terms of these error
rates.

\begin{theorem}\label{main}
For all binary classifiers $f$ and all cost-sensitive multiclass
distributions $D$, 
\[
\wreg(T(f),D) 
\le \reg(f,D')\, \sum_{n\in T} W_n,
\]
where $W_n = \E_{(x,c)\sim D}w_n(x,c)$, and
$w_n(x,{c})$ is the importance weight formed in line~\ref{line5} of
Algorithm~\ref{alg:cs-filter-tree} (the difference in 
cost between the two labels that node $n$ chooses between on $x$).
\end{theorem}

\noindent
Before proving the theorem, we state 
the corollary for multiclass classification.

\begin{corollary} \label{cor:multi}
For all binary classifiers $f$ and multiclass distributions $D$,
\[
\reg(T(f),D) \le d \, \reg(f,D'),
\]
where $d$ is the depth of the tree $T$.
\end{corollary}
\noindent
Since all importance weights are either 0 or 1, 
we don't need to apply Costing in the multiclass case.
The proof of the corollary given the theorem is simple since for any $(x,y)$,
the induced $(x,{c})$ has at most one node per level with 
induced importance weight 1; all other importance weights are 0.
Therefore, $\sum_n w_n(x,{c}) \leq d$.

Theorem~\ref{bound} provides an alternative bound for cost-sensitive 
classification.  It is the first known bound giving
a worst-case dependence of less than $k$.
\begin{theorem}
For all binary classifiers $f$ and all cost-sensitive $k$-class
distributions $D$,
\[
\wreg(T(f),D) \le k \reg(f,D')/2,
\]
where $T(f)$ and
$D'$ are as defined above.
\label{bound}
\end{theorem}

\medskip\noindent
A simple example in Section~\ref{example} shows that this bound is
essentially tight.

\bigskip\noindent
The proof of Theorem~\ref{main} uses the following folk theorem 
from \cite{costing}.

\begin{theorem}
\label{th:translate}\emph{(Translation Theorem \cite{costing})}
For any importance-weighted distribution $P$, there exists a constant
$\langle c \rangle = \E_{(x,y,c)\sim P}[c]$ such that for
any classifier $f$,
\[ 
\E_{(x,y,c)\sim P}[c \cdot \I(f(x)\neq y)] = 
\langle c \rangle
\E_{(x,y,c)\sim {P}'}[\I(f(x)\neq y)],
\]
where ${P'}(x,y,c) = \frac{c}{\langle c \rangle}P(x,y,c)$.
\end{theorem}
\noindent

Thus choosing $f$
to minimize the error rate under $P'$ is equivalent to
choosing $f$ to minimize the expected cost under $P$.
The Costing~\cite{costing} reduction 
uses rejection sampling according to the weights to
draw examples from $P'$ given examples drawn from $P$.

\noindent
The remainder of this section proves Theorems~\ref{main} and \ref{bound}.

\medskip\noindent
{\bf Proof of Theorem~\ref{main}:}~~
It is sufficient to prove the claim for any $x\in X$ because that
implies that the result holds for all expectations over $x$.

Conditioned on the value of $x$, each label $y$ has a distribution
over costs with an expected value of $\E_{{c} \sim D\mid x} [c_y]$.
The zero regret cost-sensitive classifier predicts according to $\arg
\min_y \E_{{c} \sim D\mid x} [c_y]$.  Suppose that
$T(f)$ predicts $y'$ on $x$, inducing cost-sensitive regret
\[
\wreg(y',D\mid x) = \E_{{c} \sim D\mid x} [c_{y'}] -
\min_y \E_{\vec{c} \sim D\mid x} [c_y].
\]
%The proof of the theorem is done in two steps:
\noindent
First, we 
show that the sum over the binary problems of the importance weighted
regret is at least $\wreg(y',D\mid x)$, using induction starting at the leaves.
%Then we apply the costing analysis from importance weighted binary
%classification to binary classification.
The induction hypothesis is that the sum of the regrets of importance-weighted
binary classifiers in any subtree bounds the regret of the subtree output.

For node $n$, each importance weighted binary decision between class
$a$ and class $b$ has an importance weighted regret which is either
$0$ or
$
r_n = |\E_{\vec{c} \sim D|x} [c_a - c_b] |  
= |\E_{\vec{c} \sim D|x} [c_a] - \E_{\vec{c} \sim D|x} [c_b] |,
$
depending on whether the prediction is correct or not.

Assume without loss of generality that the predictor outputs
class $b$.
The regret of the subtree $T_n$ rooted at $n$ is given by
\[r_{T_n} = \E_{\vec{c} \sim D|x} [c_b] - \min_{y \in \leaves(T_n)}
\E_{\vec{c} \sim D|x} [c_y].
\]

As a base case, the inductive hypothesis is trivially satisfied for
trees with one label.  
Inductively, assume that 
$\sum_{n'\in L}r_{n'}\geq r_L$ and $\sum_{n'\in R}r_{n'}\geq r_R$
for the left subtree $L$ of $n$ (providing $a$) and the right subtree 
$R$ (providing $b$).

There are two possibilities.  Either the
minimizer comes from the leaves of $L$ or the leaves of $R$.  The
second possibility is easy since we have
\begin{align*}
r_{T_n} &= \E_{\vec{c} \sim D|x} [c_b] - \min_{y \in
\leaves(R)} \E_{\vec{c} \sim D|x} [c_y]  
 = r_R \leq \sum_{n' \in R} r_{n'} \leq \sum_{n' \in {T_n}} r_{n'}, 
\end{align*}
proving the induction.

For the first possibility, 
\begin{align*}
r_{T_n} &= \E_{\vec{c} \sim D|x} [c_b] - \min_{y \in
\leaves(L)} \E_{\vec{c} \sim D|x} [c_y]  \\
& = \E_{\vec{c} \sim D|x} [c_b] - \E_{\vec{c} \sim D|x} [c_a] + 
\E_{\vec{c} \sim D|x} [c_a] 
 - \min_{y \in \leaves(L)} \E_{\vec{c} \sim D|x} [c_y]  \\
&= \E_{\vec{c} \sim D|x} [c_b] - \E_{\vec{c} \sim D|x} [c_a] + r_L \\
& \leq r_n + \sum_{n' \in L} r_{n'} 
\leq \sum_{n' \in {T_n}} r_{n'},
\end{align*}
which completes the induction.  The inductive hypothesis for the root
is that $ \wreg(y',D|x) \leq \sum_{n \in T} r_n$.

Using the folk theorem from~\cite{costing} (Theorem~\ref{th:translate} 
in this paper), each $r_n$ is bounded by
\[
r_n \leq W_n \reg(f_n,D'_n).
\]
Plugging this in and using Definition~(\ref{abr}), we get the theorem.
\qedblob
%\end{proof}

\bigskip\noindent
The proof of Theorem~\ref{bound} makes use of the following lemma.
Consider a filter tree $T$ on $k$ labels, 
evaluated on a cost-sensitive multiclass
example with cost vector $c\in [0,1]^k$. 
Let $S_T$ be the sum of importances over 
all nodes in $T$, and $I_T$ be the sum of importances over the nodes 
where the class with the larger cost was selected for the next round.
Let $c_T$ denote the cost of the winner chosen by $T$.

\begin{lemma} For any $c\in [0,1]^k$,
$S_{T}+c_{T}\leq I_{T}+\frac{k}{2}$.
\label{claim1}
\end{lemma}
\begin{proof}
The inequality follows by induction, the result being immediate when $k=2$.
Assume that the claim holds for the two subtrees, $L$ and $R$,
providing their respective inputs $l$ and $r$ to the root of $T$,
and $T$ outputs $r$ without loss of generality.
Using the inductive hypotheses for $L$ and $R$, we get
$S_T + c_T = S_L + S_R + |c_r - c_l| + c_r
\leq I_L + I_R + \frac{k}{2} - c_l + |c_r - c_l|$.

If $c_r \geq c_l$, we have $I_T = I_L + I_R + (c_r -c_l)$, and
\[
S_T + c_T \leq I_T + \frac{k}{2} -c_l \leq
I_T + \frac{k}{2},\] as desired.  If $c_r < c_l$, we have
$I_T = I_L + I_R$ and
$S_T + c_T \leq I_T + \frac{k}{2} - c_r \leq I_T + \frac{k}{2}$,
completing the proof.
\quad
\end{proof}

\bigskip\noindent
{\bf Proof of Theorem~\ref{bound}:}~~
Fix $(x,{c})\in \cX \times [0,1]^k$ and
take the expectation over the draw of $(x,{c})$ from $D$ as the last step.

Consider a filter tree $T$ evaluated on $(x,{c})$ using a given
binary classifier $f$.  As before, let $S_T$ be the sum of importances
over all nodes in $T$, and $I_T$ be the sum of importances over the
nodes where $f$ made a mistake.  Recall that the regret of $T$ on
$(x,{c})$, denoted in the proof by $\mbox{reg}_T$, is the difference
between the cost of the tree's output and the smallest cost $c^*$.
The importance-weighted binary regret of $f$ on $(x,{c})$ is simply
$I_T/S_T$.  Since the expected importance is upper bounded by 1,
$I_T/S_T$ also bounds the binary regret of $f$.

The inequality we need to prove is $\mbox{reg}_{T}S_{T}\leq\frac{k}{2}I_{T}$.
The proof is by induction on $k$, the result being trivial if $k=2$.
Assume that the assertion holds for the two subtrees, $L$ and $R$,
providing their respective inputs $l$ and $r$ to the root of $T$.
(The number of classes in $L$ and $R$ can be taken to be even, by
splitting the odd class into two classes with the same cost as 
the split class, which has no effect on the quantities 
in the theorem statement.)

Let the best cost $c^*$ be in the left subtree $L$.
%
%The goal is to upper bound $\mbox{reg}_TS_T$ by $\frac{k}{2}I_T$.
Suppose first ({\bf Case 1}) that $T$ chooses $r$ and $c_r>c_l$.  Let $w = c_r-c_l$.
We have $\mbox{reg}_L = c_l-c^*$ and $\mbox{reg}_T = c_r-c^* = \mbox{reg}_L + w$.
The left hand side of the inequality is thus
\begin{align*}
\mbox{reg}_TS_T &= (\mbox{reg}_{L}+w)(S_{R}+S_{L}+w) \\ 
& = w(\mbox{reg}_{L}+S_{R}+S_{L}+w)+\mbox{reg}_L(S_L+S_R) \\
& \leq w(\mbox{reg}_{L}+I_{R}+I_L -c_{r}-c_{l}+w+\frac{k}{2}) 
  +\mbox{reg}_L (I_R + I_L - c_l - c_r + \frac{k}{2})\\
& \leq \frac{k}{2}w + I_R(w+\mbox{reg}_L) + I_L(w+\mbox{reg}_L) 
+\mbox{reg}_L\left(\frac{k}{2} - c_r - c_l\right)\\
& \leq \frac{k}{2}w + I_R(w+\mbox{reg}_L) + I_L\left(w+\mbox{reg}_L + \frac{k}{2} - c_r - c_l\right)\\
& \leq \frac{k}{2}w + I_R(w+\mbox{reg}_L) + \frac{k}{2} I_L
\leq \frac{k}{2}(w + I_R + I_L) = \frac{k}{2} I_T.
\end{align*}
The first inequality follows from lemma \ref{claim1}.  The second and
fourth follow from $w(\mbox{reg}_L -c_l - c_r + w) \leq 0$.  The third
follows from $\mbox{reg}_L \leq I_L$.  The last follows from
$\mbox{reg}_T \leq \frac{k}{2}$ for $k\geq 2$.

The proofs for the remaining three cases 
($c_T = c_l < c_r$, $c_T = c_l > c_r$, and $c_l > c_r = c_T$) 
use the same machinery as the proof above.

\medskip\noindent
{\bf Case~2}:~ $T$ outputs $l$, and $c_l < c_r$.
In this case $\mbox{reg}_T = \mbox{reg}_L = c_l - c^*$. 
The left hand side can be rewritten as
\begin{align*}
\mbox{reg}_TS_T &= \mbox{reg}_{L}(S_{R}+S_{L}+c_r-c_l) 
= \mbox{reg}_{L}S_{L}+\mbox{reg}_{L}(S_{R}+c_r-c_l) \\
&\leq \mbox{reg}_L \left( I_L + I_R - 2c_l +\frac{k}{2} \right) 
\leq I_R + \mbox{reg}_L \left( I_L - 2c_l +\frac{k}{2} \right) \\
&\leq I_R + I_L \left( \mbox{reg}_L - 2c_l +\frac{k}{2} \right) 
\leq I_R + \frac{k}{2}I_L \leq \frac{k}{2}I_T.
\end{align*}
The first inequality follows from the lemma, the second from $\mbox{reg}_L
\leq 1$, the third from $\mbox{reg}_L \leq I_L$, the fourth from $-c_L
- c^* < 0$, and the fifth because $I_T = I_L + I_R$.

\medskip\noindent
{\bf Case~3}:~ $T$ outputs $l$, and $c_l > c_r$.
We have $\mbox{reg}_T = \mbox{reg}_L = c_l - c^*$.  
The left hand side can be written as
\begin{align*}
\mbox{reg}_TS_T &= \mbox{reg}_{L}(S_{R}+S_{L}+c_l-c_r) \\
& \leq \frac{|L|}{2}I_L + \mbox{reg}_{L} \left( I_R + \frac{k-|L|}{2}-c_r +c_l -c_r \right) \\
&\leq
\frac{k}{2}I_L + I_R + (c_l - 2c_r) 
\leq \frac{k}{2}(I_L + I_R + (c_l - c_r)) = \frac{k}{2}I_T,
\end{align*}
The first inequality follows from the inductive hypothesis and the
lemma, the second from $\mbox{reg}_L < 1$ and $\mbox{reg}_L < I_L$,
and the third from $c_r > 0$ and $k/2 > 1$.

\medskip\noindent
{\bf Case~4}:~ $T$ outputs $r$, and $c_l > c_r$.
Let $w = c_l-c_r$.  We have $\mbox{reg}_T = c_r - c^* = \mbox{reg}_L - w$. 
The left hand side can be written as
\begin{align*}
\mbox{reg}_TS_T &= (\mbox{reg}_{L}-w)(S_{R}+S_{L}+w) \\
&= \mbox{reg}_LS_L - wS_L + (\mbox{reg}_L-w)(S_R + w)\\
&\leq
\frac{|L|}{2}I_L - w \left( I_L + \frac{|L|}{2} -c_l \right) + 
(\mbox{reg}_L - w)\left(I_R + c_l-2c_r +\frac{k-|L|}{2} \right) \\
&\leq
\frac{|L|}{2}I_L - w \left( I_L + \frac{|L|}{2} -c_l \right) + 
(I_L - w)\frac{k-|L|}{2} \\
 & \qquad\quad + (\mbox{reg}_L - w)\left(I_R + c_l-2c_r \right) \\
&\leq 
\frac{k}{2}(I_L + I_R) - w\frac{k}{2} - w(I_L - c_l) 
 + (\mbox{reg}_L-w)(c_l-2c_r).
\end{align*}
The first inequality follows from the inductive hypothesis and the
lemma, the second from $\mbox{reg}_L \leq I_L$, and the third from
$\mbox{reg}_L \leq \frac{k}{2}$.

The last three terms are upper bounded by 
$
-w -w\mbox{reg}_L + wc_l + \mbox{reg}_Lc_l -2c_r\mbox{reg}_L -wc_l +2wc_r 
\leq -w - \mbox{reg}_L(c_r+c_l) + \mbox{reg}_Lc_l +2wc_r 
\leq -w -(c_l-c^*)c_r + wc_r +(c_l-c_r)c_r \leq 0,
$
\shrink{
\begin{align*}
-w &-w\mbox{reg}_L + wc_l + \mbox{reg}_Lc_l -2c_r\mbox{reg}_L -wc_l +2wc_r \\
&\leq -w - \mbox{reg}_L(c_r+c_l) + \mbox{reg}_Lc_l +2wc_r \\
&\leq -w -(c_l-c^*)c_r + wc_r +(c_l-c_r)c_r \leq 0,
\end{align*}
}
and thus can be ignored, yielding
$
\mbox{reg}_TS_T \leq \frac{k}{2}(I_L + I_R) = \frac{k}{2}I_T
$,
which completes the proof.
Taking the expectation over $(x,{c})$ completes the proof.
\qedblob
%\end{proof}

\subsection{Tightness of Theorem~\ref{bound}}\label{example}
The following simple example shows that the theorem is 
essentially tight.
Let $k$ be a power of two, and let every label have cost 0 if it is
is even, and 1 otherwise. The tree structure 
is a complete binary tree of depth $\log k$ with the nodes being paired 
in the order of their labels.
Suppose that all pairwise classifications are correct, except that class
$k$ wins all its $\log k$ games leading to cost-sensitive multiclass regret 1.
We have $\mbox{reg}_T=1$, 
$S_T=\frac{k}{2}+\log k -1$, and $I_T = \log k$, leading to
the regret ratio 
$
{\mbox{reg}_TS_T}/{I_T} =\Omega(\frac{k}{2\log k}),
$
almost matching the theorem's bound of 
$\frac{k}{2}$ on the ratio.

\subsection{Experimental Results}\label{S:experiments}
There is a variant of the Filter Tree algorithm, which has a
significant difference in performance in practice.  Every
classification at any node $n$ is essentially between two labels
computed at test time, implying that we could simply learn one
classifier for every pair of labels that could reach $n$ at test time.
(Note that a given pair of labels can be compared only at a single
node, namely their least common ancestor in the tree.)  The
conditioning process and the tree structure gives us a better analysis
than is achievable with the All-Pairs approach~\cite{All-Pairs}.  This
variant uses more computation and requires more data but often
maximizes performance when the form of the classifier is constrained.

We compared the performance of Filter Tree and its All-Pairs variant
described above to the performance of All-Pairs and the Tree
reduction, on a number of publicly available multiclass
datasets~\cite{UCI}.  Some datasets came with a standard training/test split:
\texttt{isolet} (isolated letter speech recognition),
\texttt{optdigits} (optical handwritten digit recognition),
\texttt{pendigits}
(pen-based handwritten digit recognition),
\texttt{satimage}, and \texttt{soybean}.
For all other datasets, we reported the average
result over 10 random splits, with $2/3$ of the dataset used for training and
$1/3$ for testing. (The splits were the same for all methods.)

\begin{figure}[t]
\includegraphics[angle=270,width=.5\textwidth]{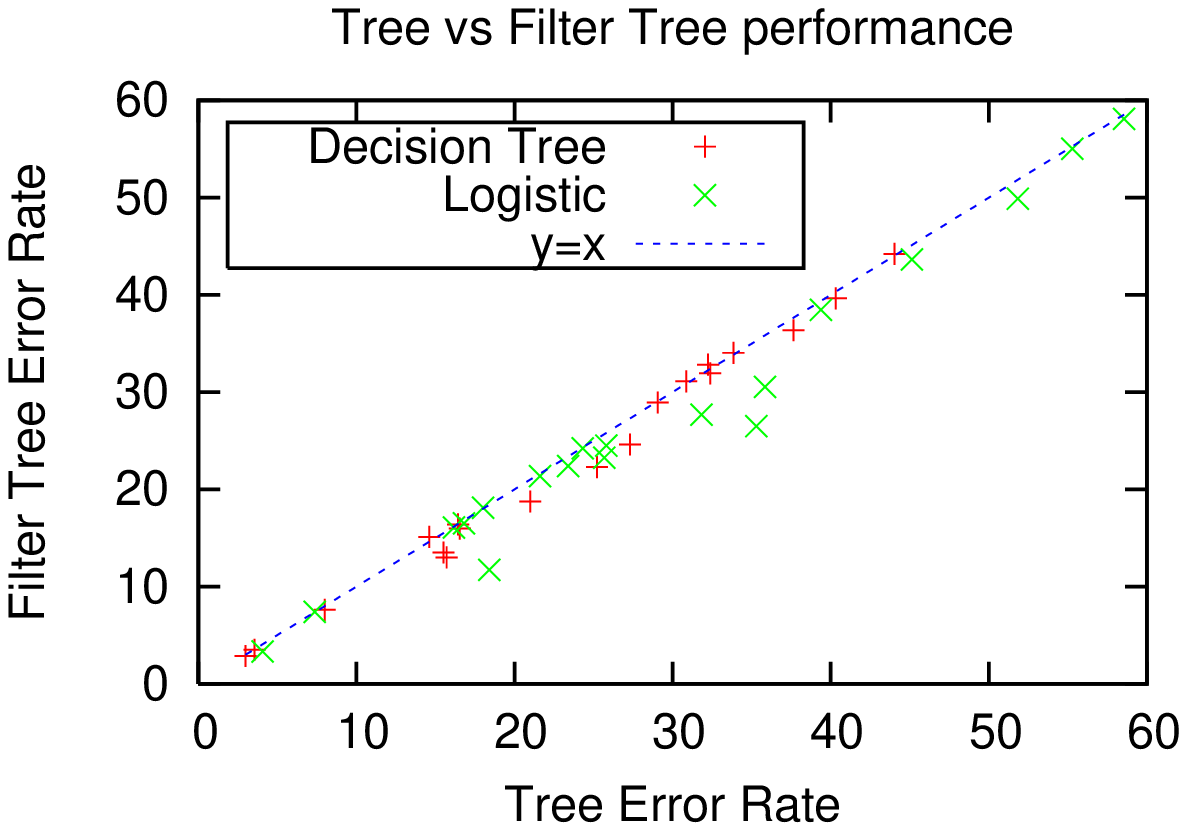}
\includegraphics[angle=270,width=.5\textwidth]{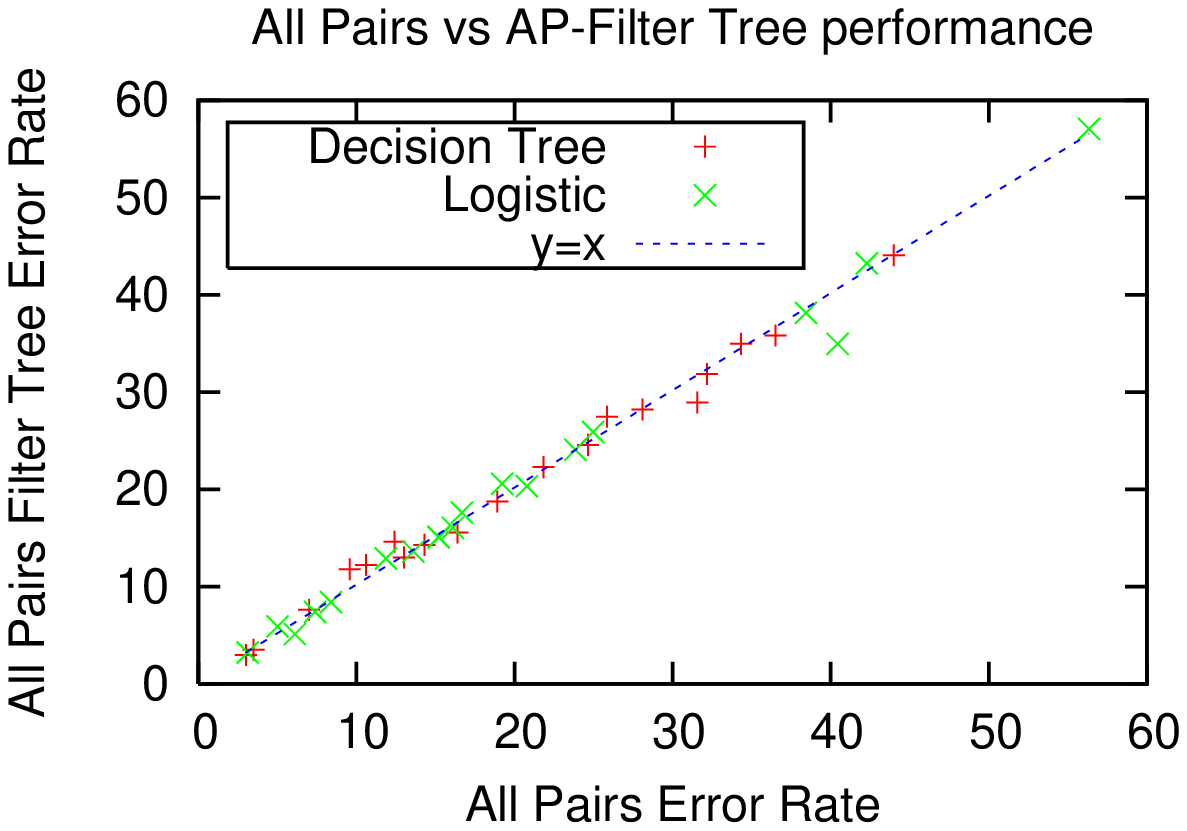}
\caption{\label{fig:comparison} Error rates (in \%) 
of \texttt{Tree} versus
\texttt{Filter-Tree} (top) and 
{All-Pairs} versus All-Pairs Filter Tree (top)
on several different datasets with
a decision tree or logistic regression classifier.}
%The Filter Tree appears to perform better in practice.}
%There is no clear dominance.}
\end{figure}

If computation is constrained and we can afford only $O(\log k)$
computation per multiclass prediction, the Filter Tree dominates the
Tree reduction, as shown in Figure~\ref{fig:comparison}.

If computation is relatively unconstrained, All-Pairs and the
All-Pairs Filter Tree are reasonable choices.  The comparison in
Figure~\ref{fig:comparison} shows that there the All-Pairs Filter Tree
yields similar prediction performance while using only $O(k)$
computation instead of $O(k^2)$.

Test error rates using decision trees (J48) and logistic regression as
binary classifier learners are reported in Table~\ref{T1}, using 
Weka's implementation with default parameters~\cite{weka}.
The lowest error rate in each row is shown in bold, although
in some cases the difference is insignificant.

\section{Error-Correcting Tournaments}
\label{sec:Multi-Elimination}
In this section, we extend filter trees to 
$m$-elimination tournaments, also called $(m-1)$-error-correcting tournaments.
As this section builds on 
Sections~\ref{S:algorithm} and \ref{S:analysis}, understanding them is required 
before reading this section.  For simplicity, we work with only the
multiclass case.  An extension for cost-sensitive multiclass problems
is possible using the importance weighting techniques of the previous
section.

\subsection{Algorithm Description}
An {\it $m$-elimination tournament} operates in two phases. 

The first phase consists of $m$
single-elimination tournaments over the $k$ labels where a label is
paired against another label at most once per round.  Consequently,
only one of these single elimination tournaments has a simple binary
tree structure; see, for example, Figure~\ref{fig:me_phase1} for an
$m=3$ elimination tournament on $k=8$ labels.  There is substantial
freedom in how the pairings of the first phase are done; our
bounds depend on the depth of any mechanism which 
pairs labels in $m$ distinct single elimination tournaments.  One such
explicit mechanism is given in~\cite{Min_find}.  Note that once an
example has lost $m$ times, it is eliminated and no longer
influences training at the nodes closer to the root.

The second phase is a final elimination phase, where we select the
winner from the $m$ winners of the first phase. It consists of a
redundant single-elimination tournament, where the degree of
redundancy increases as the root is approached.  To quantify the
redundancy, let every subtree $Q$ have a \emph{charge} $c_Q$ equal to
the number of leaves under the subtree.  First phase winners at the
leaves of the final elimination tournament have charge $1$.  For any
non-leaf node comparing the outputs of subtrees $A$ and $B$, the importance
weight of a binary example created at the node is set to either $c_A$ or $c_B$,
depending on whether the label comes from $B$ or $A$. 
In tournament applications, an importance weight can be
expressed by playing games repeatedly where the winner of $A$ must
beat the winner of $B$ $c_B$ times to advance, and vice versa.
When the two labels compared are
the same, the importance weight is set to $0$,
indicating there is no preference in the pairing amongst the two
choices.  

\begin{figure}
\begin{center}
\includegraphics[angle=0,width=.35\textwidth]{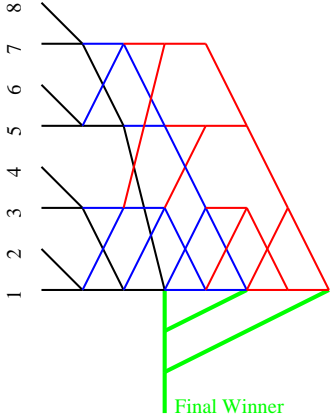}
\end{center}
\caption{\label{fig:me_phase1} An example of a $3$-elimination
  tournament on $k=8$ players.  There are $m=3$ distinct single
  elimination tournaments in first phase---one in black, one in blue,
  and one in red.  After that, a final elimination phase occurs over
  the three winners of the first phase.  The final elimination
  tournament has an extra weighting on the nodes, detailed in the text. }
\end{figure}

\subsection{Error Correcting Tournament Analysis}
%We first analyze the computation and then the regret, as in the
%previous section.  
A key concept throughout this section is the
{\it importance depth}, defined as the worst-case length
(number of games) of the overall tournament, where importance-weighted 
matches in the final elimination phase are played as repeated games.
In Theorem~\ref{thm:Importance-depth} we prove a bound on the
importance depth.

The computational bound per example is essentially just the importance
depth.

\begin{theorem}\label{thm:structural-depth} \emph{(Structural Depth Bound)} 
For any $m$-elimination tournament, the training and test computation
is $O(m + \ln k)$ per example.
\end{theorem}

\begin{proof}
The proof is by simplification of the importance depth bound (theorem
\ref{thm:Importance-depth}), which bounds the sum of importance weights at
all nodes in the tournament.

To see that the importance depth controls the computation, first note that
the importance depth bounds the tournament depth since all importance weights
are at least 1.  At training time, any one example is used at most once
per tournament level starting at the leaves.  At testing time, an
unlabeled example can have its label determined by traversing the
structure from root to leaf.
\end{proof}

\subsection{Regret analysis}
Our regret theorem is the analogue of Corollary~\ref{cor:multi} for
error-correcting tournaments, and the notation is as defined there.
As in the previous section, the reduction transforms a multiclass
distribution $D$ into an induced distribution $D'$ 
over binary labeled examples.
As before, $T(f)$ denotes the multiclass classifier
induced by a given binary classifier $f$ and tournament structure $T$.

It is useful to have the notation $\left
\lceil m \right \rceil _2$ for the smallest power of
$2$ larger than or equal to $m$.

\begin{theorem}\label{multi-theorem} \emph{(Main Theorem)} 
\noindent
For all distributions $D$ over $k$-class examples, all binary classifiers $f$,
all $m$-elimination tournaments $T$, the ratio
of ${\reg(T(f),D)}$ to ${\reg(f,D')} $ is upper
bounded by
\[
\begin{cases}
2 + \frac{\lceil m \rceil_2}{m} + \frac{k}{2m} &\text{for all $m \geq 2$ and $k> 2$} \\
4 + \frac{2\ln k}{m}+2\sqrt{\frac{\ln k}{m}} &
\text{for all $k \leq 2^{62}$ and $m \leq 4\log_{2}k$} \\
\end{cases}
\]
\end{theorem}
\noindent
The first case shows that a regret ratio of $3$ is achievable for very
large $m$.  The second case is the best bound for cases of common
interest.  For $m = 4 \ln k$ it gives a ratio of $5.5$.

\medskip
\begin{proof}
The proof holds for each input $x$, and hence in expectation over $x$.

Fix $x$, and let $p_y = D(y\mid x)$ for $y\in \{1,\ldots, k\}$.
We can define the regret of any label $y$ as $r_y = 
p^* - p_y$, where $p^*=\max_{a \in \{1 ,\cdots, k \}} p_a$.

The regret of a node $n$ comparing labels $a$ and $b$ from subtrees $A$ and $B$,
and outputting $a$, is \[r_n = c_B(p_b-p_a)_+,\] where we use the predicate
$(z)_+=\max(z,0)$.  Thus $r_n$ is 0 if $n$ outputs the more likely label.
If $n$ is in a first phase tournament, $r_n = (p_b-p_a)_+$.

Finally, the regret of a subtree $T$ is defined as $r_T =
\sum_{n\in T} r_n$.

The first part of the proof is by induction on the tree structure $F$
of the final phase.  The invariant for a subtree $Q$ of $F$ won by
label $a$ is
\[c_Q r_a \leq r_Q + \sum_{w \in \leaves(Q)} r_{W},\] 
where $w$ is
the winner of a first phase single-elimination tournament $W$.

When $Q$ is a leaf $w$ of $F$, we have $c_Q r_w = r_w \leq r_{W}$,
where the inequality is from Corollary~\ref{cor:multi} noting that the depth
of $W$ times the average regret over the nodes in $W$ is $r_W$.

Assume inductively that the hypothesis holds at node $n$ 
comparing labels $a$ and $b$ from subtrees $A$ and $B$,
and outputting $a$:
$
c_A r_a
\leq r_A + \sum_{w \in \leaves(A)} r_{W}$ 
and
$c_B r_b \leq r_B + \sum_{w \in \leaves(B)} r_{W}.$
We have
$r_Q + \sum_{w \in \leaves(Q)} r_{W} \geq r_n + c_Ar_a + c_Br_b$
by the inductive hypothesis.  

Now, there are two cases:  Either $p_b \leq p_a$, in which case 
$r_n=0$ and $c_Ar_a + c_Br_b \geq c_Ar_a + c_Br_a \geq C_Qr_a$, as desired.  
Or $p_b > p_a$, in which case $r_n = c_B(p_b-p_a)$ and thus
\begin{align*}
r_n + c_Ar_a + c_Br_b &= c_Bp_b-c_Bp_a + c_Ap^*-c_Ap_a + c_Bp^* -c_Bp_b \\
&= p^*c_Q -p_aC_Q = (p^*-p_a)C_Q = r_aC_Q,
\end{align*}
finishing the induction.

Finally, letting $y$ be the prediction of $T(f)$ on $x$,
\[
m \reg(T(f),D\mid x) = c_F r_{y}
\leq r_F + \sum_{w \in \leaves(F)} r_{W} 
\leq d \reg(f,D'\mid x)),
\]
where $d$ is the maximum importance depth.
Applying the importance depth
theorem~(Theorem \ref{thm:Importance-depth}) and algebra completes the proof.
\end{proof}

\bigskip\noindent
The depth bound follows from the following three lemmas.

\begin{lemma}\label{lem:fpdb} \emph{(First Phase Depth bound)} 
The importance depth of the first phase tournament is bounded by the
minimum of
\[
\begin{cases}
\lceil \log_2 k\rceil + m \lceil \log_2 (\lceil \log_2 k \rceil + 1)\rceil \\
1.5 \lceil \log_2 k \rceil + 3 m + 1 \\
\left\lceil \frac{k}{2} \right\rceil + 2 m \\
\text{For $k\leq2^{62}$ and $m\leq4\log_{2}k$,}\ 2(m-1)+\ln k+\sqrt{\ln k}\sqrt{\ln k+4(m-1)}.
\end{cases}
\]
\end{lemma}
\begin{proof}
The depth of the first phase is bounded by the classical problem of robust
minimum finding with low depth.  The first three cases hold because any
such construction upper bounds the depth of an error-correcting
tournament, and one such construction has these
bounds~\cite{Min_find}.

For the fourth case, we construct the depth bound by analyzing a
continuous relaxation of the problem. 
%, similar to the style of~\cite{Abernethy}.  
The relaxation allows the number of labels
remaining in each single elimination tournament of the first phase to
be broken into fractions.  Relative to this version, the actual
problem has two important discretizations:
\begin{enumerate}
\item When a single-elimination tournament has only a single label remaining,
  it enters the next single elimination tournament.  This can have the
  effect of \emph{decreasing} the depth compared to the continuous
  relaxation.
\item When a single-elimination tournament has an odd number of labels
  remaining, the odd label does not play that round.  Thus the number
  of players does not quite halve, potentially \emph{increasing} the
  depth compared to the continuous relaxation.
\end{enumerate}
In the continuous version, tournament $i$ on round $d$ has $\frac{{d \choose
  {i-1}}k}{2^d}$ labels, where the first tournament corresponds to
$i=1$.  Consequently, the number of labels remaining in any of the
tournaments is $\frac{k}{2^d} \sum_{i=1}^{m} {d \choose {i-1}}.$ We
can get an estimate of the depth by finding the value of $d$ such that
this number is 1.

This value of $d$ can be found using the Chernoff bound.  The
probability that a coin with bias $1/2$ has $m-1$ or fewer heads in
$d$ coin flips is bounded by $m^{- 2d \left(\frac{1}{2} -
  \frac{m-1}{d} \right)^2 }$, and the probability that this occurs in
$k$ attempts is bounded by $k$ times that.  Setting this value to $1$,
we get $ \ln k = 2d \left(\frac{1}{2} - \frac{m-1}{d} \right)^2.  $
Solving the equation for $d$, gives $
%0 & = 4(m-1)^2 - [4(m-1) + 2 \ln k] d + d^2 \\
%d & = 2(m-1) + \ln k + \sqrt{(2(m-1) + \ln k)^2 - 4 (m-1)^2}\\
d = 2(m-1) + \ln k + \sqrt{ 4(m-1) \ln k + (\ln k)^2 }$.
This last formula was verified computationally for $k < 2^{62}$ and $m
< 4 \log_2 k$ by discretizing $k$ into factors of $2$ and running a simple
program to keep track of the number of labels in each tournament at
each level.  For $k\in \{2^{l-1}+1, 2^l\}$, we used a pessimistic value
of $k = 2^{l-1} + 1$ in the above formula to compute the bound, and
compared it to the output of the program for $k=2^l$.\quad
\end{proof}

\begin{lemma}\label{lem:spdb} \emph{(Second Phase Depth Bound)} 
In any $m$-elimination tournament, the second phase has importance
depth at most $\left\lceil m \right\rceil_2 - 1 $ rounds for $m>1$.
\end{lemma}
\begin{proof}
When two labels are compared in round $i\geq 1$, the importance weight
of their comparison is at most $2^{i-1}$.  Thus we have
$ \sum_{i=1}^{\lceil \log_2 m \rceil - 1} 2^{i-1} + \lfloor m
\rfloor_2 = \lceil m \rceil_2 - 1$.
\end{proof}

\bigskip\noindent
Putting everything together gives the importance depth theorem.
\begin{theorem}\label{thm:Importance-depth} \emph{(Importance Depth Bound)} 
For all $m$-elimination tournaments, the importance depth is
upper bounded by
\[
\begin{cases}
\lceil \log_2 k\rceil + m \lceil \log_2 (\lceil \log_2 k \rceil + 1)\rceil + \lceil m \rceil_2 \\
1.5 \lceil \log_2 k \rceil + 3 m + \lceil m \rceil_2 \\
\left\lceil \frac{k}{2} \right\rceil + 2 m + \lceil m \rceil_2 \\
\text{For $k\leq2^{62}$ and $m\leq4\log_{2}k$,}\ 2m + \lceil m
  \rceil_2 + 2\ln k+2\sqrt{m\ln k}.
\end{cases}
\]
\end{theorem}

\begin{proof}
We simply add the depths of the
first and second phases from Lemmas~\ref{lem:fpdb} and
\ref{lem:spdb}.
For the last
case, we bound $\sqrt{\ln k+4(m-1)}
\leq \sqrt{\ln k} + 2\sqrt{m}$ and eliminate subtractions in Lemma
\ref{lem:spdb}.
\end{proof}

\section{Lower Bound}
\label{sec:LB}
All of our lower bounds hold for a somewhat more powerful adversary
which is more natural in a game playing tournament setting.  In
particular, we disallow reductions which use importance weighting on
examples, or equivalently, all importance weights are set to $1$.
Note that we can modify our upper bound to obey this constraint by
transforming final elimination comparisons with importance weight $i$
into $2 i -1$ repeated comparisons and use the majority vote.  This
modified construction has an importance depth which is at most $m$
larger implying the ratio of the adversary and the reduction's regret
increases by at most $1$.

The first lower bound says that for any reduction algorithm $B$, there
exists an adversary $A$ with the average per-round regret $r$ such
that $A$ can make $B$ incur regret $2r$ even if $B$ knows $r$ in
advance.  Thus an adversary who corrupts half of all outcomes can
force a maximally bad outcome.
In the bounds below, $f_B$ denotes the multiclass classifier induced
by a reduction $B$ using a binary classifier $f$.

\begin{theorem}
For any deterministic reduction $B$ from $k>2$ classification 
to binary classification,
there exists a choice of $D$ and $f$ such that
$\reg(f_B,D) 
\geq 2 \reg(f,B(D))$.
\end{theorem}

\begin{proof}
The adversary $A$ picks any two labels $i$ and $j$. All comparisons 
involving $i$ but not $j$, are decided in favor of $i$.  Similarly for $j$.
The outcome of comparing $i$ and $j$ is determined by the
parity of the number of comparisons between $i$ and $j$ in some fixed
serialization of the algorithm.  If the parity is odd, $i$ wins; 
otherwise, $j$ wins.  The outcomes of all other comparisons are picked 
arbitrarily.

Suppose that the algorithm halts after some number of queries $c$
between $i$ and $j$.  If neither $i$ nor $j$ wins, the adversary can
simply assign probability $1/2$ to $i$ and $j$.  The adversary pays
nothing while the algorithm suffers loss 1, yielding a regret ratio of
$\infty$.

Assume without loss of generality that $i$ wins.  The depth of the
tournament is either $c$ or at least $c+1$, because each label can appear
at most once in any round.  If the depth is $c$, then since $k>2$,
some label is not involved in any query, and the adversary can set the
probability of that label to $1$ resulting in $\v(B)=\infty$.

Otherwise, $A$ can set the probability of label $j$ to be $1$ while
all others have probability $0$.  The total regret of $A$ is at most
$\lfloor \frac{c+1}{2} \rfloor$, while the regret of the winning label
is $1$.  Multiplying by the depth bound $c+1$, gives a regret ratio
of at least $2$.
\end{proof}

\bigskip\noindent
Note that the number of rounds in the above bound can depend on $A$.
Next, we show that for any algorithm $B$ taking the same number of
rounds for any adversary, there exists an adversary $A$ with a regret
of roughly one third, such that $A$ can make $B$ incur the maximal loss,
even if $B$ knows the power of the adversary.

\begin{lemma}
For any deterministic reduction $B$ to binary classification with
number of rounds independent of the query outcomes, there exists a
choice of $D$ and $f$ such that
$\reg(f_B,D) 
\geq (3 - \frac{2}{k}) \reg(f,{B}(D))$.
\end{lemma}
\begin{proof}
Let $B$ take $q$ rounds to determine the winner, for any set of query
outcomes. We will design an adversary $A$ 
with incurs regret $r = \frac{qk}{3k-2}$, such that 
$A$ can make $B$ incur the maximal loss of 1, even
if $B$ knows $r$. 

The adversary's query answering strategy is to answer consistently
with label $1$ winning for the first $\frac{2(k-1)}{k}r$ rounds,
breaking ties arbitrarily.  The total number of queries that $B$ can
ask during this stage is at most $(k-1)r$ since each label can play at
most once in every round, and each query occupies two labels.  Thus
the total amount of regret at this point is at most $(k-1)r$, and
there must exist a label $i$ other than label $k$ with at most $r$
losses.  In the remaining $q-\frac{2(k-1)}{n}r = r$ rounds, $A$
answers consistently with label $i$ and all other skills being 0.

Now if $B$ selects label $1$, $A$ can set $D(i\mid x)=1$ with $r/q$
average regret from the first stage.  If $B$ selects label $i$
instead, $A$ can choose that $D(1\mid x)=1$.  Since the number of queries
between labels $i$ and $k$ in the second stage is at most $r$, the
adversary can incurs average regret at most $r/q$.  If $B$ chooses any
other label to be the winner, the regret ratio is unbounded.
\end{proof}

\appendix
\section{Table with Experimental Results}
\begin{table*}
\begin{small}
\centering
\begin{tabular}{|l|r|r|r|r||r|r|r|r||r|r|} \hline
& \multicolumn{4}{|c||}{J48} & \multicolumn{4}{|c||}{Logistic Regression} 
\\ \hline
{Dataset}$(k)$ & 
{Tree} & {FT} & {AP} & {APFT} &
{Tree} & {FT} & {AP} & {APFT} \\ \hline
%----------------------------------------------------------------------
   arrhythmia (13) & 37.64 & 36.37      & {\bf 34.32}  & 34.97            
    	      & 55.27 & 55.04 	   & 40.44 	  & {\bf 34.97} 
\\
%----------------------------------------------------------------------
   audiology (24)  & 32.37 & 31.93      & {\bf 28.08}  & 28.21 
              & 31.83 & 27.69 	   & {\bf 24.98}  & 25.90   
\\
%----------------------------------------------------------------------
   ecoli (8)     & 21.00 & {\bf 18.75}& 18.90        & {\bf 18.75}      
	      & 18.00 & 18.10 	   & 15.20  & {\bf 15.06}
\\
%----------------------------------------------------------------------
   flare (7)     & 16.42 & 16.38      & 16.38        & {\bf 15.57}      
	      & 16.17  & 16.07 & 16.09 & {\bf 16.03}
\\
%----------------------------------------------------------------------
   glass (6)     & 33.84 & 34.02      & 32.18        & {\bf 31.86}      
	      & 39.37 & 38.46 & 38.43 & {\bf 38.13}
\\
%----------------------------------------------------------------------
   isolet (26)    & 27.30 & 24.60      & {\bf 12.40}  & 14.60            
	      & 35.30  & 26.50 & {\bf 8.40} & {\bf 8.40}
\\
%----------------------------------------------------------------------
   kropt (18)     & 40.32 & 39.66      & 36.50        & {\bf 35.81}      
	      & 58.55 & 58.09 & {\bf 56.34} & 57.06
\\
%----------------------------------------------------------------------
   letter (25)    & 16.53 & 15.96      & {\bf 9.58}   & 11.77            
	      & 51.84 & 49.89 & {\bf 16.66} & 17.62
\\
%----------------------------------------------------------------------
   lymph (4)     & 25.22 & 22.28      & {\bf 21.83}  & 22.28            
	      & 24.32 & 24.20 & {\bf 23.86} & 24.07
\\
%----------------------------------------------------------------------
nursery (5)      & 3.55  & {\bf 3.49} & {\bf 3.49}   & {\bf 3.49}       
	      & {\bf 7.36} & 7.41 & 7.39 & 7.39 
\\
%----------------------------------------------------------------------
optdigits (10)    & 15.50  & 13.50       & {\bf 10.60}   & 12.20             
	      & 18.40 & 11.70  & {\bf 5.00} & 5.90
\\
%----------------------------------------------------------------------
page-blocks (5)  & 2.99  & {\bf 2.84} & 3.00         & 2.95         
	      & 4.06 & 3.31 & {\bf 3.12} & 3.21
\\
%----------------------------------------------------------------------
pendigits (10)    & 8.00   & 7.60       & {\bf 7.00}    & 7.60             
		& 23.40 &  22.40  & 6.10 &  {\bf 5.10} 
\\
%----------------------------------------------------------------------
satimage (6)     & 14.60 & 15.10      & {\bf 14.30}  & {\bf 14.30}      
		& 25.80  & 24.50   & 15.20  & {\bf 15.10} 
\\
%----------------------------------------------------------------------
soybean (19)      & 15.70  & {\bf 13.00} & {\bf 13.00}   & \bf{13.00}       
		& 16.80 & 16.50   & {\bf 13.60} & {\bf 13.60}
\\
%----------------------------------------------------------------------
vehicle (4)       & 30.86 & 31.11      & 31.57        & {\bf 28.93}      
		& 21.60 & 21.37 & 20.78 & {\bf 20.31}
\\
%----------------------------------------------------------------------
vowel (11)        & 29.06 & 28.92      & 24.64        & {\bf 24.57}      
		& 35.85 & 30.53 & {\bf 11.85}  & 12.90
\\
%----------------------------------------------------------------------
yeast (10)        & 44.04 & 44.21      & {\bf 43.99} & 44.06  
		& 45.13 & 43.66 & {\bf 42.28} & 43.26
\\
%----------------------------------------------------------------------
\hline
\end{tabular}
\end{small}
\caption{
Test error rates (in \%) using J48 and logistic regression as binary learners. 
AP and FT stand for All-Pairs and Filter Tree respectively.
APFT is the All-Pairs variant of the Filter Tree.
}
\label{T1}
\end{table*}
\end{document}